\def\year{2021}\relax
\documentclass[letterpaper]{article} 
\usepackage{aaai21}  
\usepackage{times}  
\usepackage{helvet} 
\usepackage{courier}  
\usepackage[hyphens]{url}  
\usepackage{graphicx} 
\usepackage{natbib}  
\usepackage{caption} 
\usepackage[linesnumbered,ruled,vlined]{algorithm2e}
\usepackage{booktabs}
\usepackage{multirow}
\usepackage{comment}
\usepackage{xcolor}
\usepackage{cite}
\usepackage{enumitem}
\usepackage{subcaption}

\usepackage{ifthen}

\usepackage{amsmath}
\usepackage{amssymb}
\usepackage{amsthm}
\newtheorem{theorem}{Theorem}
\newtheorem{lemma}[theorem]{Lemma}

\newcommand{\CB}{\mathbb{C}}
\newcommand{\CA}{\mathcal{A}}
\newcommand*{\myfont}{\fontfamily{pcr}\selectfont}

\title{S$^*$: A Heuristic Information-Based Approximation Framework for Multi-Goal Path Finding}
\author {
    K. Chour,\textsuperscript{\rm 1}
    S. Rathinam,\textsuperscript{\rm 1}
    R. Ravi \textsuperscript{\rm 2}\\
}
\affiliations {
    \textsuperscript{\rm 1} Mechanical Engineering, Texas A \& M University, College Station \\
    \textsuperscript{\rm 2} Tepper School of Business, Carnegie Mellon University, Pittsburgh \\
    ckennyc@tamu.edu, srathinam@tamu.edu, ravi@cmu.edu
}

\begin{document}

\maketitle

\begin{abstract}


We combine ideas from uni-directional and bi-directional heuristic search, and approximation algorithms for the Traveling Salesman Problem, to develop a novel framework for a Multi-Goal Path Finding (MGPF) problem that provides a 2-approximation guarantee. MGPF aims to find a least-cost path from an origin to a destination such that each node in a given set of goals is visited at least once along the path. 
We present numerical results to illustrate the advantages of our framework over conventional alternates in terms of the number of expanded nodes and run time.


\end{abstract}

\noindent 

\newboolean{shortver}
\setboolean{shortver}{false}

\section{Introduction}


Multi-Goal Path Finding (MGPF) aims to find a least-cost path in a graph $G=(V,E)$ with non-negative edge costs such that the path starts from an origin ($s\in V$) and ends at a destination ($d\in V$), and each node in a given set of goals ($\bar{T}\subseteq V$) is visited at least once along the path. In the special case when the goal set is empty ($\bar{T}=\emptyset$), MGPF reduces to the least-cost path problem and is polynomial time solvable \cite{dijkstra1959note,lawler2001combinatorial}. For the general case, we must also determine the sequence in which the goals must be visited, and therefore, MGPF is a generalization of the Steiner\footnote{Any node that is \emph{not} required to be visited is referred to as a \emph{Steiner node}. A path may choose to visit a Steiner node if it helps in either finding feasible solutions or reducing the cost of travel.} Traveling Salesman Problem~\cite{SteinerTSP} and is NP-Hard. 
MGPF arises in numerous aerial robot and logistics applications as discussed in recent surveys~\cite{UAVsurvey,macharet_campos_2018}. 

\begin{figure*}
    \centering
    \includegraphics[scale=0.41]{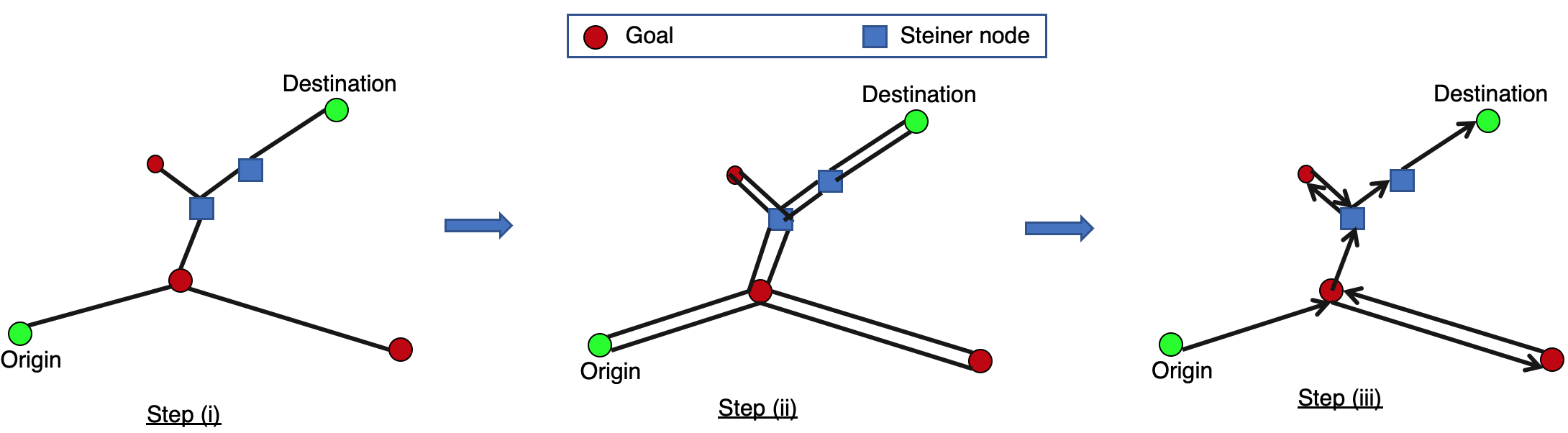}
    \caption{An approximation algorithm for MGPF: (i) Construct a suitable Steiner tree, (ii) Double the edges in the Steiner tree to form an Eulerian graph and (iii) Find a path in the Eulerian graph from the origin to destination that visits each goal at least once, while discarding remaining edges.} 
    \label{fig:SteinerApprox}
\end{figure*}

Existing 2-approximation algorithms for the MGPF and its variants \cite{Kou1981,Mehlhorn1988} rely on three steps (Fig. \ref{fig:SteinerApprox}): (i) find a suitable Steiner tree spanning the origin, goals and the destination, (ii) double the edges in the Steiner tree to obtain an Eulerian graph, and then finally (iii) find a path in  the Eulerian graph that is feasible for the MGPF problem. The 2-approximation ratio and the computational complexity of these algorithms primarily relies on the Steiner tree construction in step (i); this construction must be done so that the cost of the Steiner tree constructed is at most the optimal cost of the MGPF. Well known primal-dual~\cite{Ravi1994,AKR,goemans1997primal} or minimum spanning tree based algorithms~\cite{Kou1981,Mehlhorn1988} can be used to find such a Steiner tree.


The objective of this article is to propose a new approximation framework that fuses existing methods for Steiner tree construction with heuristic information to develop new algorithms for step (i) of the approximation algorithm for MGPF (Fig. \ref{fig:SteinerApprox}). As a consequence, this framework provides new efficient 2-approximation algorithms for MGPF. Our work follows the spirit of the A* \cite{hart1968aFormalBasis} (or the bi-directional \cite{Pohl:1969}) heuristic search methods where a best-first search procedure from the origin (or the origin and the destination) was combined with heuristic information to develop new algorithms for the least-cost path problem. In fact, in the special case when the goal set is empty, the primal-dual algorithm~\cite{AKR} for the Steiner tree problem, depending on how it is applied, reduces to either the uni-directional search \cite{dijkstra1959note} or bi-directional search \cite{Nicholson} algorithm available for the least-cost path problem. Therefore, one can view our work in this article as a direct generalization of the A* and the bi-directional heuristic search procedures to Steiner tree computation and to MGPF. 

We refer to the proposed framework as Steiner$^*$ (S$^*$) and present its two variants. In the first variant, we use A$^*$ to grow closed and open sets from each node in $T:=\{s,d\}\bigcup \bar{T}$ and simultaneously construct a Steiner tree when relevant bounding conditions are satisfied. We refer to this variant as S$^*$-unmerged since we do not merge the closed sets corresponding to distinct nodes in $T$ even when they overlap with each other. The second variant is referred to as S$^*$-merged where the closed sets are merged when appropriate bounding conditions are satisfied akin to what happens in the Kruskal's minimum spanning tree algorithm \cite{kruskal}. The S$^*$-merged framework is agnostic to the underlying optimality conditions used for the least-cost path computations during the search process; specifically, one can use the optimality conditions from A* \cite{hart1968aFormalBasis} or bi-directional search \cite{Nicholson} or Meet in the Middle (MM) \cite{MM} algorithms in the S$^*$-merged framework and guarantee the required properties. We note here that while there are several bi-directional heuristic search methods \cite{Champeaux1983,KWA1989,Eckerle1994,Kaindl1997,Barker2015,MM,NBS}, we use MM \cite{MM} as a representative bi-directional search method for least-cost path computations in our framework as our goal is to address the MGPF; other methods will be considered in the future. Like uni-directional and bi-directional heuristic search, the expectation here is that combining existing algorithms with heuristic information will reduce the number of expanded nodes and possibly the computation time (Fig. \ref{fig:expectation}).\ifthenelse{\boolean{shortver}}{%
\footnote{Counterexamples to this expectation is presented in the appendix of \cite{chour2021s}.}}

\begin{figure}
    \centering
    \includegraphics[scale=0.35]{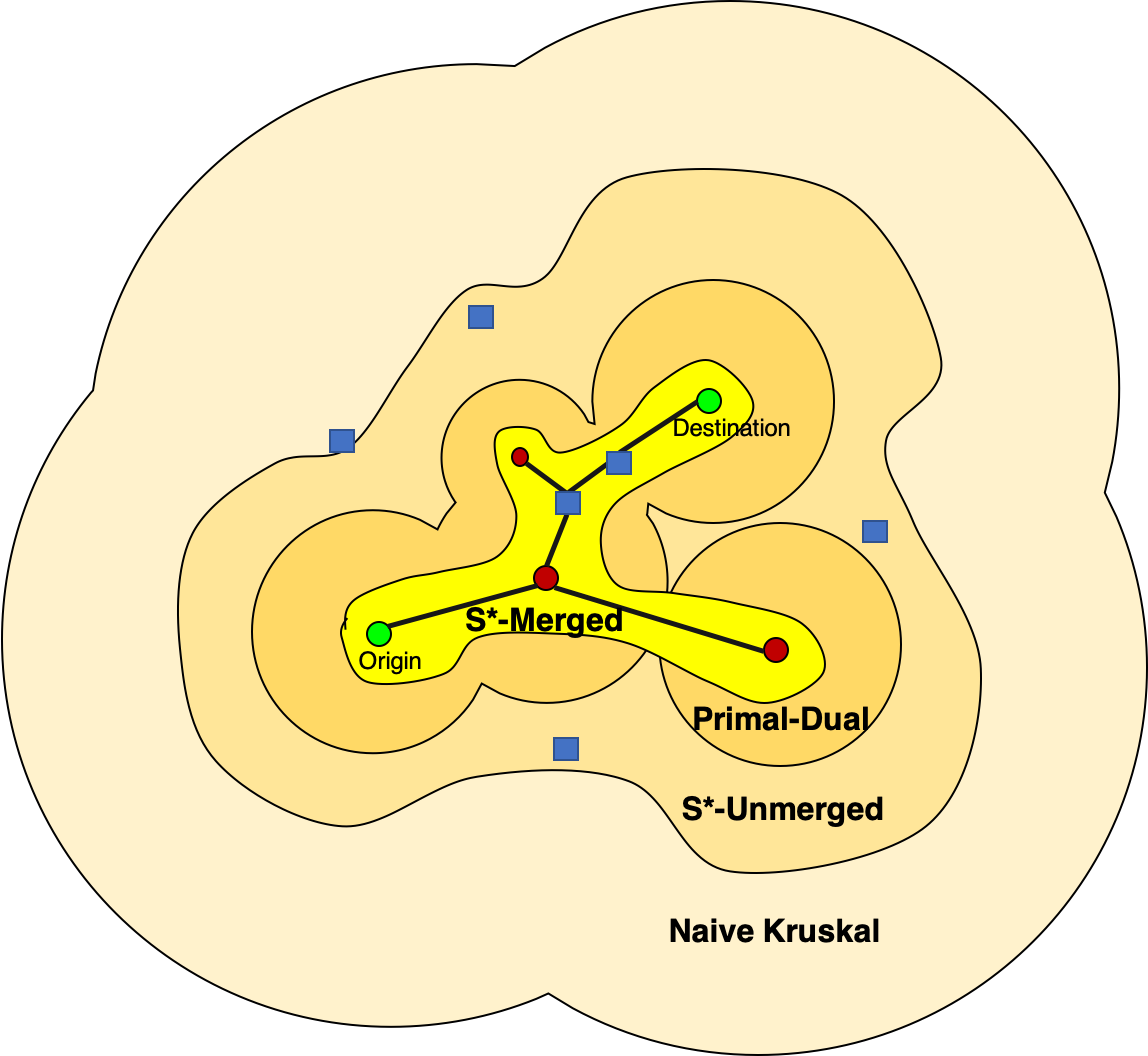}
    \caption{Each shaded region shows our expectation on the set of expanded nodes for conventional solvers and the proposed framework. Here, naive Kruskal is the popular approach \cite{Kou1981} described in the Background and Preliminaries section.}
    \label{fig:expectation}
\end{figure}

After describing the new algorithms with theoretical performance guarantees for the Steiner tree construction, we provide extensive computational results on the performance on the proposed framework for instances derived from the Multi-Agent Path Finding (MAPF) library\footnote{\urlstyle{tt}\url{ https://movingai.com/benchmarks}} \cite{stern2019multi}. While these numerical results clearly illustrate the benefits of the proposed framework in several scenarios, we do not claim that the proposed framework is superior to conventional solvers for each and every instance of the MGPF problem. Nevertheless, the proposed framework is the first of its kind for MGPF and provides a new line of research for related problems.

\section{Background and Preliminaries}\label{sec:background}


Let $c(u,v)\geq 0$ denote the cost of the edge joining two distinct vertices $u$ and $v$ in $G=(V,E)$. The cost of a path is defined as the sum of the edges in the path. Let $cost^*(u,t)$ denote the cost of the least-cost path from $u$ to $t$ in $G$. Let $\bar{h}_t(u)$ be an underestimate on $cost^*(u,t)$. To simplify our presentation and proofs, we assume $\bar{h}_t(u)$ is obtained using a consistent heuristic~\cite{hart1968aFormalBasis}.

A Steiner tree is a connected subgraph of edges that spans a subset of relevant nodes $T =\{s,d\}\bigcup \bar{T} \subseteq V$ also commonly referred to as terminals. Finding a Steiner tree which minimizes the sum of the cost of the edges in the tree is NP-Hard~\cite{Karp72}. Therefore, a popular approach for finding a suitable Steiner tree in step (i) of the approximation algorithm (Fig. \ref{fig:SteinerApprox}) is to find a Minimum Spanning Tree (MST) in the metric completion\footnote{The metric completion of the terminals is a complete weighted graph on the terminals $T$ where the weight of an edge between a pair of terminals is the minimum cost of a path between them in $G$.} of the terminals, and then replace each edge in the MST by the corresponding least-cost path in $G$. There are several implementations of this approach \cite{Kou1981, Mehlhorn1988, Ravi1994, goemans1997primal}. 
Irrespective of the specific implementation used, the Steiner tree construction relies on satisfying the following key properties:

\begin{itemize}[leftmargin=.8cm]
    \item[{\bf (SP)}] Ensure that when a path between a pair of terminals in $T$ is confirmed by a path finding algorithm, it is indeed a least-cost path in $G$ between them.
    \item[\bf (K)] When a path $P$ between two terminals is included in the Steiner tree, it obeys Kruskal's condition \cite{kruskal} for inclusion in the MST of the metric completion of $T$. This requires that all paths between any pair of terminals with costs lower than the cost of $P$ have been considered, and $P$ does not create any cycles when added to the current tree.
\end{itemize}
If a Steiner tree algorithm satisfies the above key properties, then it is known that the cost of the Steiner tree obtained using the algorithm is at most equal to the optimal MGPF cost which leads to a 2-approximation algorithm for MGPF~\cite{Kou1981, Mehlhorn1988}. The variants of the S* framework ensures that these two key properties are maintained thus proving the approximation guarantee of the final MGPF solution obtained using this approach.

\section{S*-unmerged}

\begin{algorithm}[t!]
    \SetAlgoLined
	\textbf{Inputs:} \\
	$G=(V,E)$, $c(u,v) \ \forall u,v \in V$, $T\subseteq V$\\
	$\bar{h}_t(u) \ \forall t \in T, u \in V$ \tcp{consistent lower bounds on $c^*(u,t)$}
	\textbf{Output:}\\
	$S_T$ \tcp{Steiner tree spanning T}
	\textbf{Initialization:}\\
    $C_t := \{t\} \ \forall t \in T$ \tcp{Closed sets} 
    $O_t:= \{u: (u,t)\in E\} \ \forall t \in T$ \tcp{Open sets} 
    $D_t := T\setminus \{t\}  \ \forall t \in T$ \tcp{Destination sets} 
    $g_t(u):=c(t,u) ~\forall u\in V, t\in T$ \\
    $f_t(u):=g_t(u)+h(u,D_t),~\forall u\in V,t\in T $ \\
    $Q=\emptyset$  \tcp{Paths eligible for $S_T$} 
    $S_T := \emptyset$\\
    \vspace{.1cm}
    {\bf Main Loop:} \\
	\While{all the terminals are not connected in $S_T$}
 	{   
 	    $u_t:= \arg\min_{u\in O_t} f_t(u) ~\forall t\in T$ \tcp{$t$ nominates best node}
        $t^* = \arg \min \{f_t(u_t): t\in T\}$ \tcp{Choose nominator with least $f$ cost}
         $C_{t^*}:=C_{t^*}\cup \{u_{t^*}\}$; 
         $O_{t^*}:=O_{t^*}\setminus \{u_{t^*}\}$ \\
         \For{$v\in \{v: (v,u_{t^*})\in E, v\notin C_{t^*} \}$}
         {  $O_{t^*}:=O_{t^*}\cup \{v\}$ \\
         $g_{t^*}(v):= \min(g_{t^*}(v),g_{t^*}(u_{t^*})+c(u_{t^*},v))$\\
         $f_{t^*}(v):= g_{t^*}(v) + h(v,D_{t^*}), \ \forall v \in O_{t^*}$\\
         } 
          \If{$u_{t^*}\in D_{t^*}$}{
             $Q:=Q\cup \{PATH(t^*,u_{t^*})\}$ \color{blue}\\
             \textcolor{blue}{$D_{t^*}:= D_{t^*} \setminus \{u_{t^*}\}$ \label{unm-prio-begin}\\
            $f_{t^*}(v):= g_{t^*}(v) + h(v,D_{t^*}), \ \forall v \in O_{t^*}$\\
            $D_{u_{t^*}}:= D_{u_{t^*}} \setminus \{t^*\}$\\
            $f_{u_{t^*}}(v):= g_{u_{t^*}}(v) + h(v,D_{u_{t^*}}), \ \forall v \in O_{u_{t^*}}$\label{unm-prio-end}
             }
             }
       $S_T:=UpdateSteinerTree(S_T,Q,f)$ 
 	}
 	\textbf{return } $S_T$
 	\caption{{\myfont S\textsuperscript{*}-unmerged}}
 	\label{Alg:unmerged}
\end{algorithm}

\begin{algorithm}[h!]
    \SetAlgoLined
    {$f^* := \min_{t} \min_{u\in O_t} f_{t}(u)$ \\
    $Q' = Q$ \tcp{Process paths locally}
	\While{$Q'$ is nonempty}{
	Choose a path $p\in Q'$ with the cheapest cost joining components $C_1,C_2\in S_T$ \\
	\If{adding $p$ to $S_T$ does not form a cycle AND $cost(p) \leq f^*$}
	    {Add $p$ to $S_T$\\
         $\bar{C} := C_1 \cup C_2$\\
	     \For{$t\in \bar{C}\cap T$}{$D_t:= T \setminus \bar{C}$ \label{unm-essential}\\
	     $f_{t}(v):= g_{t}(v) + h(v,D_{t}),$ $\forall v\in O_{t}$  
	     }
	     }
	   	 Delete $p$ from $Q'$

    }
    \textbf{return} $S_T$ 
 	\caption{ $UpdateSteinerTree(S_T,Q,f)$}
 	\label{Alg:updateSteinetree}
 	}
\end{algorithm}

{\myfont S\textsuperscript{*}-unmerged} (Algorithm \ref{Alg:unmerged}) uses A$^*$ to build closed and open sets from each terminal in $T$. We borrow the usual definitions of $f$, $g$ and $h$ costs from A$^*$; however, in this framework, each terminal maintains its own version of the $f$, $g$ and $h$ costs for each of the nodes in its closed and open sets. Each terminal $t\in T$ maintains a destination list $D_t$ which includes all the terminals not yet connected to $t$ in the Steiner tree. In Algorithm \ref{Alg:unmerged}, for any $S\subset V$, $h(u,S)$ is defined as the underestimate from node $u$ to reach any terminal in $S$, $i.e.$, $h(u,S) := \min\{\bar{h}_t(u): t\in T \cap S\}$. 

After the initialization step, during each iteration of the the main loop, {\myfont S\textsuperscript{*}-unmerged} proceeds to let each terminal nominate a node with the least $f$-cost from its open set. The best nominated node $u_{t^*}$ with the smallest $f$ value is then moved to the corresponding terminal($t^*$)'s  closed set. Next, $u_{t^*}$ is expanded and the corresponding $g$ and $f$ costs of its neighbors are updated (lines 19-23 in Algorithm \ref{Alg:unmerged}). If paths are confirmed between two distinct terminals, they are also added to $Q$ (line 25 in Algorithm \ref{Alg:unmerged}) which maintains a list of paths eligible for the Steiner tree construction. Blue lines 26-29 in Algorithm \ref{Alg:unmerged} are referred as {\it re-prioritization} steps and are not mandatory; however, they may help in reducing the total number of expanded nodes during the search process at the expense of additional computation time. Finally, S*-unmerged checks (line 31 in Algorithm \ref{Alg:unmerged}) if the Steiner tree ($S_T$) needs to be updated based on the changes in $Q$ or the bounding function $f$. 

The procedure in Algorithm \ref{Alg:updateSteinetree} ensures paths are added to $S_T$ only if they satisfy the property (K). 
First, we consider the minimum $f$-value ($f^*$) among all the nodes in the open sets of all the terminals. Note that this will be $f$-cost of the next best nominated node in the algorithm. Therefore, all confirmed paths with costs at most equal to $f^*$ from all terminals have been explored by now.
We then process the confirmed paths in $Q$ locally in {\it increasing order of cost} in the following way: 
\begin{itemize}
    \item If a path is between two terminals that are not connected in $S_T$ and its cost is at most $f^*$ (line 5 of Algorithm \ref{Alg:updateSteinetree}), we include it in $S_T$. We also update the destination sets of the terminals due to changes in $S_T$ and correspondingly change the $f$-costs for nodes in the open sets of these terminals (lines 8-11 of Algorithm \ref{Alg:updateSteinetree}). 
    \item If a path doesn't satisfy the conditions in line 5 of Algorithm \ref{Alg:updateSteinetree},  we ignore and delete it locally since it does not obey property (K).
\end{itemize}

\subsection{Correctness of S$^*$-unmerged}

First, we observe that the open and closed sets of terminals are updated in the same manner as the $A^*$ algorithm. 
Furthermore, 
as the heuristic costs are consistent, $g_t(u)$ for any node $u\in C_t$ is equal to $cost^*(t,u)$. Moreover, the shortest paths found between $t$ and any node in $C_t$, and their corresponding $g$ costs remain valid even after the changes in $D_t$ since we recompute the lower bounds $h$ to remain consistent, and use them in updating the bounds on the open sets. This shows that the paths finalized by the algorithm obey the (SP) property. 
As a result of the condition used in Algorithm~\ref{Alg:updateSteinetree}, S$^*$-unmerged also satisfies the (K) property. 
Hence, we have the following theorem.

\begin{theorem}\label{thm:unmerged}
S*-unmerged finds a Steiner tree of cost at most equal to the optimal MGPF cost. This Steiner tree can then be used to obtain a 2-approximation algorithm for MGPF. 
\end{theorem}

\section{S*-merged}

\begin{algorithm}
    \SetAlgoLined
    \textbf{Input:}\\
	$G=(V,E)$, $c(u,v) \ \forall u,v, \in V$, $T\subseteq V$\\
	$\bar{h}_t(u) \ \forall t \in T, u \in V$ \tcp{consistent lower bounds on $c^*(u,t)$} 
	\textbf{Output:}\\
	$S_T$ \tcp{Steiner tree spanning T}
    \textbf{Initialization:}\\
    $\CB :=\{\{t\},~\forall t\in T\}$\\
    $C_\CA:=\CA,~\forall \CA \in \CB$ \tcp{Closed sets of $\CA$}
    $O_\CA:=\{u:(u,t) \in E, t\in \CA\}, ~\forall \CA \in \CB$ \tcp{Open sets of $\CA$}
    $D_\CA:=T\setminus \CA$, $\forall \CA\in \CB$\tcp{Destination sets of $\CA \in \CB$}
    $g_\CA(u):=\min\{c(t,u):~t\in \CA, t\in T\},~\forall u\in V, \CA \in \CB$ \\
    $f_\CA(u):=g_\CA(u) + h(u, D_\CA),~\forall u\in V, \CA \in \CB$ \tcp{Lower bound on $cost^*(u,D_\CA)$}
    $Q=\emptyset$ \tcp{Paths eligible for $S_T$}
    $S_T=\emptyset$\\
    \vspace{0.3cm}
    \textbf{Main Loop:}\\
    \While{all the terminals are not connected in $S_T$}
    {
       $u_\CA=\arg\min_{\mathrm{u\in O_\CA}} f_\CA(u)~\forall \CA\in \CB$ \tcp{$\CA$ nominates best node}
       $\CA^*=\arg\min\{f_\CA(u_\CA):\CA \in \CB\}$ \tcp{Choose nominator with least $f$ cost}
       $C_{\CA^*} = C_{\CA^*} \cup \{u_{\CA^*}\}$;
       $O_{\CA^*} = O_{\CA^*} \setminus \{u_{\CA^*}\}$\\
        \For{$v\in \{v:(v,u_{\CA^*})\in E, v\notin C_{\CA^*}\}$}
         {
                $O_{\CA^*} = O_{\CA^*}\cup \{v\}$, \\
                $g_{\CA^*}(v) = \min\{g_{\CA^*}(v), g_{\CA^*}(u_{\CA^*}) + c(u_{\CA^*},v)\}$\\
                $f_{\CA^*}(v) = g_{\CA^*}(v) + h(v,D_{\CA^*})$        
          }
       \For{$\CA \in \CB, \CA\neq \CA^*$}
       {\If{\textcolor{red}{Path Confirmation Condition between $\CA^*$ and $\CA$ is satisfied}\label{test}}
       {
       If $terminal_{\CA}(u)$ denotes a terminal in $\CA$ that is nearest to $u$, 
       $Q:=Q\cup PATH^*(terminal_\CA(u),terminal_{\CA^*}(u))$\\ \color{blue} 
       \textcolor{blue}{$D_{\CA^*}:=D_{\CA^*}\setminus \{t:t\in \CA\cap T\}$\label{merged-prio-begin}\\
        $\forall v\in O_{\CA^*}, f_{\CA^*}(v) = g_{\CA^*}(v) + h(v,D_{\CA^*})$\\
       $D_{\CA}:=D_{\CA}\setminus \{t:t\in \CA^*\cap T\}$\\
       $\forall v\in O_{\CA}, f_{\CA}(v) = g_{\CA}(v) + h(v,D_{\CA})$\label{merged-prio-end}
       }

        }
    }
    Let $\bar{\CB}$ denote all the info pertaining to $\CB$\\
    $[S_T,\bar{\CB}] = UpdateSteinerTree\_Merge(S_T,Q,\bar{\CB})$ 
    }
       
    \caption{{\myfont S\textsuperscript{*}-merged}}
    \label{Alg:S*-merged}
\end{algorithm}

\begin{algorithm}[h!]
    \SetAlgoLined
    $f^* = \max \{  \min_{\CA} (\min_{u\in O_\CA} f_{\CA}(u)),\min_{\CA \neq \CA'} (rmin_{\CA} + rmin_{\CA'}) \}$\\
    $Q' = Q$ \tcp{Process paths locally}
	\While{$Q'$ is nonempty}
	{
	Choose a path $p\in Q'$ with the cheapest cost joining components $C_1,C_2\in S_T$ \\
	\If{adding $p$ to $S_T$ does not form a cycle AND $cost(p) \leq f^*$}
	    {
	    Suppose $p$ connects components $\CA_1,\CA_2\in \CB $\\
	    $\bar{\CB}=ComponentMerge(\CA_1,\CA_2,\bar{\CB})$\\
	    }
	    Delete $p$ from $Q'$
 	}
 	\textbf{return} $[S_T,\bar{\CB}]$ 
 	\caption{ $UpdateSteinerTree\_Merge(S_T,Q,\bar{\CB})$}
 	\label{Alg:MupdateSteinetree}
\end{algorithm}

\begin{algorithm}[t]
    \SetAlgoLined
    $\CA_{12} = \CA_1 \cup \CA_2$ \tcp{merge components}
    $D_{\CA_{12}} = D_{\CA_1}\cup D_{\CA_2}\setminus\{t:t\in \CA_{12}\cap T\}$\label{merged-essential}\\
    $OC_{1}:=O_{\CA_1} \cup C_{\CA_1}$, $OC_{2}:=O_{\CA_2} \cup C_{\CA_2}$  \\
    $g_{\CA_{12}}(u) = \min\{g_{\CA_1}(u),g_{\CA_2}(u) \}$ for all $u\in OC_1\cup OC_2$
    \tcp{merge g costs depending on the set it is in}
    $C_{\CA_{12}} = (C_{\CA_1}\cup C_{\CA_2})\setminus \{u: g_{\CA_1}(u) < g_{\CA_2}(u),~ u\in O_{\CA_1}\cap C_{\CA_2} \bigvee g_{\CA_2}(u) < g_{\CA_1}(u),~ u\in O_{\CA_2}\cap C_{\CA_1}\}$ \tcp{Remove nodes from the closed set if the g cost is lower in the open sets}
    $O_{\CA_{12}} = (OC_1\cup OC_2)\setminus C_{\CA_{12}} $\\
    $f_{\CA_{12}}(u) = g_{\CA_{12}}(u) + h(u, D_{\CA_{12}}) ~\forall u\in O_{\CA_{12}}$ \tcp{update fcosts}
    Remove $\bar{A}_1,\bar{A}_2$ from $\bar{\CB}$ and add $\bar{\CA_{12}}$ to $\bar{\CB}$ \\
    \textbf{return} $\bar{\CB}$ 
 	\caption{$ComponentMerge(\CA_1,\CA_2,\bar{\CB})$}
 	\label{Alg:Framework}
\end{algorithm}


Rather than carry out search from single terminals, {\myfont S\textsuperscript{*}-merged} (Algorithm \ref{Alg:S*-merged}) keeps track of the connectivity structure among the terminals in $S_T$ using a set $\CB$ of components which are initialized to singleton terminals. When two components merge, we simply merge the set of terminals in these components in $\CB$. We also carefully extend the definition of open and closed sets to subsets of terminals and ensure we update them so that they obey the conditions that the $g$-values of nodes in the closed sets give optimal paths from the node to some terminal in the component, and that the $f$-values remain lower bounds on reaching a terminal in another component. For this we will also need to ensure that the destination set of the merged components are updated appropriately.

Unlike {\myfont S\textsuperscript{*}-unmerged} where each terminal nominates its best node, in {\myfont S\textsuperscript{*}-merged}, each component nominates a node with the least $f$-cost from its open set. The best nominated node $u_{\CA^*}$ with the least $f$ value is then moved to the corresponding components (${\CA^*}$) closed set. Next, $u_{\CA^*}$ is expanded and the corresponding $g$ and $f$ costs of its neighbors are updated (lines 20-24 in Algorithm \ref{Alg:S*-merged}). If paths are confirmed between terminals in two distinct components, they are also added to $Q$ (line 27 in Algorithm \ref{Alg:S*-merged}). Similar to {\myfont S\textsuperscript{*}-unmerged}, the re-prioritization steps in blue lines 28-31 of Algorithm \ref{Alg:S*-merged} are not mandatory and can be used to speed up the implementation as needed. Finally, {\myfont S\textsuperscript{*}-merged} checks (line 35 in Algorithm \ref{Alg:S*-merged}) if the Steiner tree ($S_T$) and the component structure needs to be updated based on changes in $Q$ or the bounding functions of these components. 

We derive three versions of {\myfont S\textsuperscript{*}-merged} based on the method used to confirm the least-cost paths between components (red line 26 in Algorithm \ref{Alg:S*-merged}). These methods are drawn from three well-known variants, namely the bidirectional Heuristic Search (HS) with the $fmin$ rule~\cite{Pohl:1969}, bidirectional Best-first Search (BS) with the $gmin$-based rule~\cite{Nicholson} which also mimics the classic primal-dual algorithms~\cite{AKR,goemans1997primal}, and Meet-in-the-Middle (MM) Search~\cite{MM}. These versions are correspondingly referred to as S$^*$-HS, S$^*$-BS and S$^*$-MM. The following discussion presents the path criterion used in each of them.

\begin{itemize}
    \item {\bf Path confirmation criterion for S$^*$-HS:} In this version, we check if there is a node $u$ such that the sum of the $g$-values of the shortest paths to $u$ from two different components $\CA$ and $\CA^*$ is at most the larger of the lower bounds for reaching any terminal in the destination sets for $\CA$ and $\CA^*$; in other words, we test if $\min_{u\in V} (g_{\CA^*}(u) + g_\CA(u)) \leq \max(f_{\CA^*}(u_{\CA^*}),f_{\CA}(u_\CA))$. If this is the case, the path we have found via $u$ represents a least-cost path between $\CA$ and $\CA^*$. This stopping condition is also commonly referred to as the {\it ``fmin condition''} for bidirectional heuristic search (Bi-HS)~\cite{Pohl:1969,survey-AAAI18}. 

    \item {\bf Path confirmation criterion for S$^*$-BS:} Let $gmin_{\CA}$ denote the smallest $g$-value among the nodes in the open set of $\CA$. Note that this is the node in the open set that can be confirmed next according to Djikstra's algorithm. We can then use the sum of the values of $gmin_{\CA}$ and $gmin_{\CA^*}$ to check if there is a least-cost path between any terminal in $\CA$ and any terminal in $\CA^*$: $\min_{u\in V} (g_{\CA^*}(u) + g_\CA(u)) \leq gmin_{\CA^*}+gmin_{\CA}$. This is exactly the stopping condition to confirm a path in bidirectional best-first search~\cite{Nicholson,survey-AAAI18} which ensures that property (SP) holds for paths confirmed using this rule. Using this criterion also reduces S$^*$-BS to the conventional primal-dual algorithm \cite{AKR} for the Steiner tree problem. 
    
    \item {\bf Path confirmation criterion for S$^*$-MM:} For MM, we need more definitions. Define $c_{min}$ to be the minimum cost of any edge in the graph. Let the priority of a node $u$ for component $\CA$ be defined as $pr_{\CA}(u) = \max \{f_{\CA}(u), 2g_{\CA}(u)\}$ where the first term denotes a lower bound on the cost to any other component and the second is twice the confirmed cost of connecting a terminal in the component to node $u$. Now, let $prmin_{\CA} = \min_{u \in O_{\CA}} pr_{\CA}(u)$ for any $\CA$. When a pair of components $\CA$ and $\CA^*$ are evaluated for a path between them, we define $C = \min \{ prmin_{\CA}, prmin_{\CA^*} \}$. We can now use the path criterion from MM \cite{MM} to confirm a least-cost path between any terminal in $\CA$ and any terminal in $\CA^*$ as follows: $\min_{u\in V} (g_{\CA^*}(u) + g_\CA(u)) \leq \max \{ C,f_{\CA^*}(u_{\CA^*}),f_{\CA}(u_\CA),gmin_{\CA^*}+gmin_{\CA}+c_{min} \}$. This ensures property (SP) holds for paths confirmed using this version.
    
    
\end{itemize}

The procedure in Algorithm \ref{Alg:MupdateSteinetree}, similar to the {\it UpdateSteinerTree} procedure in Algorithm \ref{Alg:updateSteinetree}, ensures paths are added to $S_T$ only if they satisfy the property (K). A key difference in Algorithm \ref{Alg:MupdateSteinetree} is the addition of new bounds to $f^*$ to ensure different versions of {\myfont S\textsuperscript{*}-merged} can be handled efficiently. To do this, for a component $\CA$, we first define $rmin_{\CA}$ as the minimum $g$-value over all nodes in the boundary of $\CA$, namely those nodes in its closed set with a neighbor in its open set. Intuitively, if we draw a ball of this radius around the terminals in $\CA$, every boundary node will occur only at this distance or later, so if we drew such balls around two different components $\CA$ and $\CA'$, they would be disjoint. We then generalize the definition of $f^*$ used in the Steiner Tree updating algorithm as follows: $f^* = \max \{  \min_{\CA} (\min_{u\in O_\CA} f_{\CA}(u)), \min_{\CA \neq \CA'} (rmin_{\CA} + rmin_{\CA'}) \}$. By the disjointness of these two balls represented by the last term, we can see that using this definition to pick paths satisfies property (K).

\subsection{Correctness of S$^*$-merged}

Since the path confirmation criteria for these three algorithms are directly drawn from the stopping conditions in the corresponding Bi-HS, Bi-BS and MM algorithms, it follows that the three algorithms obey the (SP) property when they confirm paths between components.

The main point of difference in the merged methods from regular source-destination path-finding algorithms is the definition of open and closed sets since they are now for components rather than just the source or destination. 
But this is precisely what is handled in the careful redefinition of these sets for a merged component in Algorithm~\ref{Alg:Framework}. In particular, when components $\CA$ and $\CA'$ merge, if a node is present in the current closed sets of both $\CA$ and $\CA'$, we use the smaller of the two confirmed $g$-estimates for the shortest path to it. However, if it is present in the closed set of $\CA$ and the open set of $\CA'$ but the $g$-estimate is smaller to $\CA'$, then we remove it from the closed set of the merged component since we have a potentially better path from $\CA'$ and since it is still in the open set and not confirmed for its shortest path to $\CA'$. The open set of the merged component is simply those nodes in the union of the open sets of both merging components that are not retained in the closed set. Once the $f$ and $g$ costs of the merged components are updated correctly, it also follows that the update Steiner tree method in Algorithm \ref{Alg:MupdateSteinetree} ensures all the three versions of {\myfont S\textsuperscript{*}-merged} satisfy property (K). This leads to the following theorem. 

\begin{theorem}\label{thm:merged}
The {\myfont S\textsuperscript{*}-merged} framework when specialized to any of the three path confirmation criteria ({\myfont S\textsuperscript{*}-HS}, {\myfont S\textsuperscript{*}-BS}, {\myfont S\textsuperscript{*}-MM}) finds a Steiner tree of cost at most equal to the optimal MGPF cost. This Steiner tree can then be used to obtain a 2-approximation algorithm for MGPF. 
\end{theorem}

\section{Numerical Results}

\begin{figure*}[t]
\centering
\includegraphics[scale=0.81]{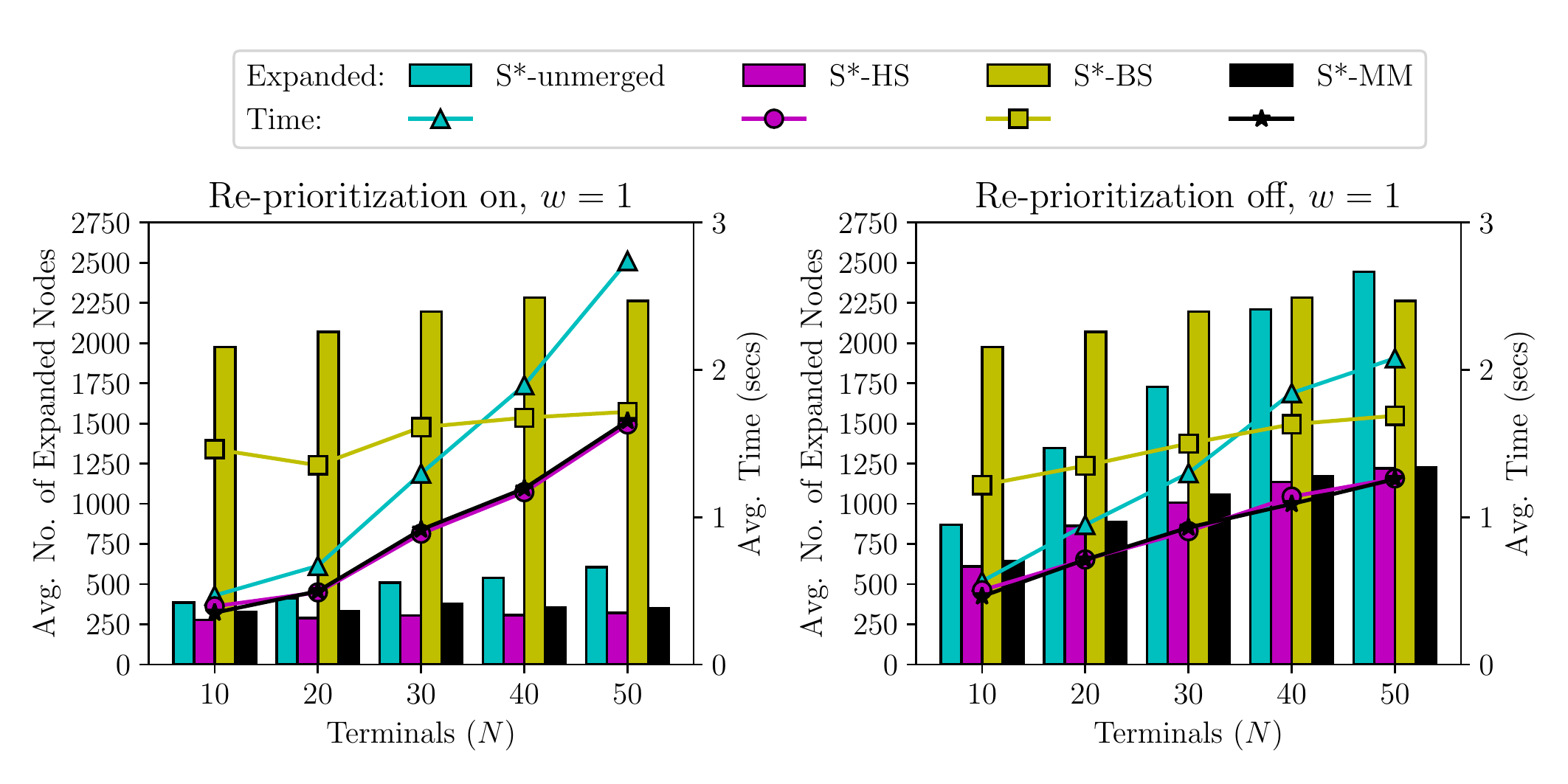} 
\caption{``den312d" map: Average number of expanded nodes and computational times as a function of number of terminals and re-prioritization. $w$ is fixed at 1. For these instances, Naive Kruskal expands between $22005-119805$ nodes with runtimes in $4.01-21.27$ secs as $N$ varies from 10 to 50 terminals.}
\label{fig:res_terminals}
\end{figure*}

\textbf{Setup:} Computational experiments were conducted on a computer with a 2.80 GHz Intel Core i7-7700HQ processor. All algorithms were implemented in Python 3.6 under Ubuntu 18.04. We compared the number of expanded nodes and runtimes of the proposed algorithms, namely {\myfont S\textsuperscript{*}-unmerged}, {\myfont S\textsuperscript{*}-HS}, and {\myfont S\textsuperscript{*}-MM}, against two conventional solvers, the primal-dual (or {\myfont S\textsuperscript{*}-BS}) and the naive Kruskal's approach\footnote{Here, we implement the approach described in the Background and Preliminaries section. First, we compute the least-cost paths between any pair of terminals to find the metric completion. Then use Kruskal's algorithm to find a MST for the metric completion.}. 
Each of the algorithms was evaluated on five separate 8-neighbor type grid maps, obtained via the MAPF benchmark library. These maps (see Table \ref{tab:summary}) were chosen based on the shape of the obstacles (maze or randomized) or their absence. Within each map, a varying number of terminals ({$ N=10,20,30,40,50$}) was randomly generated and placed. For each map and $N$, 10 problem instances were generated. Comparisons were also made with respect to factors such as merging, reprioritization, and heuristic strengths. Due to space constraints, we first present the results for the ``den312d" map in the MAPF library with and without the re-prioritization steps; later, in Table \ref{tab:summary}, we present results for all the maps for a fixed number of terminals with no re-prioritization. 
\vspace{.1cm}

\noindent\textbf{Heuristics via landmarks:} Each map was pre-processed to provide a fast look-up table for heuristic lower-bound estimates between any pair of nodes in the map. Each of these heuristic estimates was then scaled by a weighting factor $w$ to understand its impact on the overall performance of the algorithms. For small maps, heuristic estimates were obtained by computing the least-costs between any pair of nodes in the map using Dijkstra's algorithm. However, for moderately-sized maps, estimates were obtained using the ALT method \cite{goldberg2005computing} in combination with the octile distance. To implement the ALT method, one hundred ``landmarks" were randomly chosen throughout the maps such that the landmarks were ``border nodes" in the graph (with node degree $<$ 8). Dijkstra's algorithm was then used to find the least-cost from each landmark to the remaining nodes in the map; these least-costs were in turn used to compute a lower bound on the least-cost between any pair of nodes in the map. 


\subsection{Comparisons based on expanded nodes and time}

Fig. \ref{fig:res_terminals} shows the average number of expanded nodes and computational time (in secs) as a function of the number of terminals for all the algorithms. 
The weighting factor $w$ for the heuristics in these results was set to 1. The three heuristic-based algorithms ({\myfont S\textsuperscript{*}-unmerged}, {\myfont S\textsuperscript{*}-HS}, and {\myfont S\textsuperscript{*}-MM}), expanded fewer nodes on average than the conventional solvers which did not use any heuristic information. The merged algorithms ({\myfont S\textsuperscript{*}-HS}, {\myfont S\textsuperscript{*}-MM}) outperformed the others in terms of expanded nodes. This trend was consistent across all maps and weighting factors (see Table 
\ref{tab:summary}). On the other hand, with respect to average computation times, the primal-dual algorithm ({\myfont S\textsuperscript{*}-BS}) was competitive in comparison to the other merged versions and {\myfont S\textsuperscript{*}-unmerged} on the tested instances (this can also be observed in Table \ref{tab:summary}). These runtimes were also dependent on whether the re-prioritization steps (both in {\myfont S\textsuperscript{*}-unmerged} and {\myfont S\textsuperscript{*}-merged}) were switched on or off. This will be examined in the next subsection.

\subsection{Impact of re-prioritization}


\begin{figure*}[h!]
\centering
\includegraphics[scale=.81]{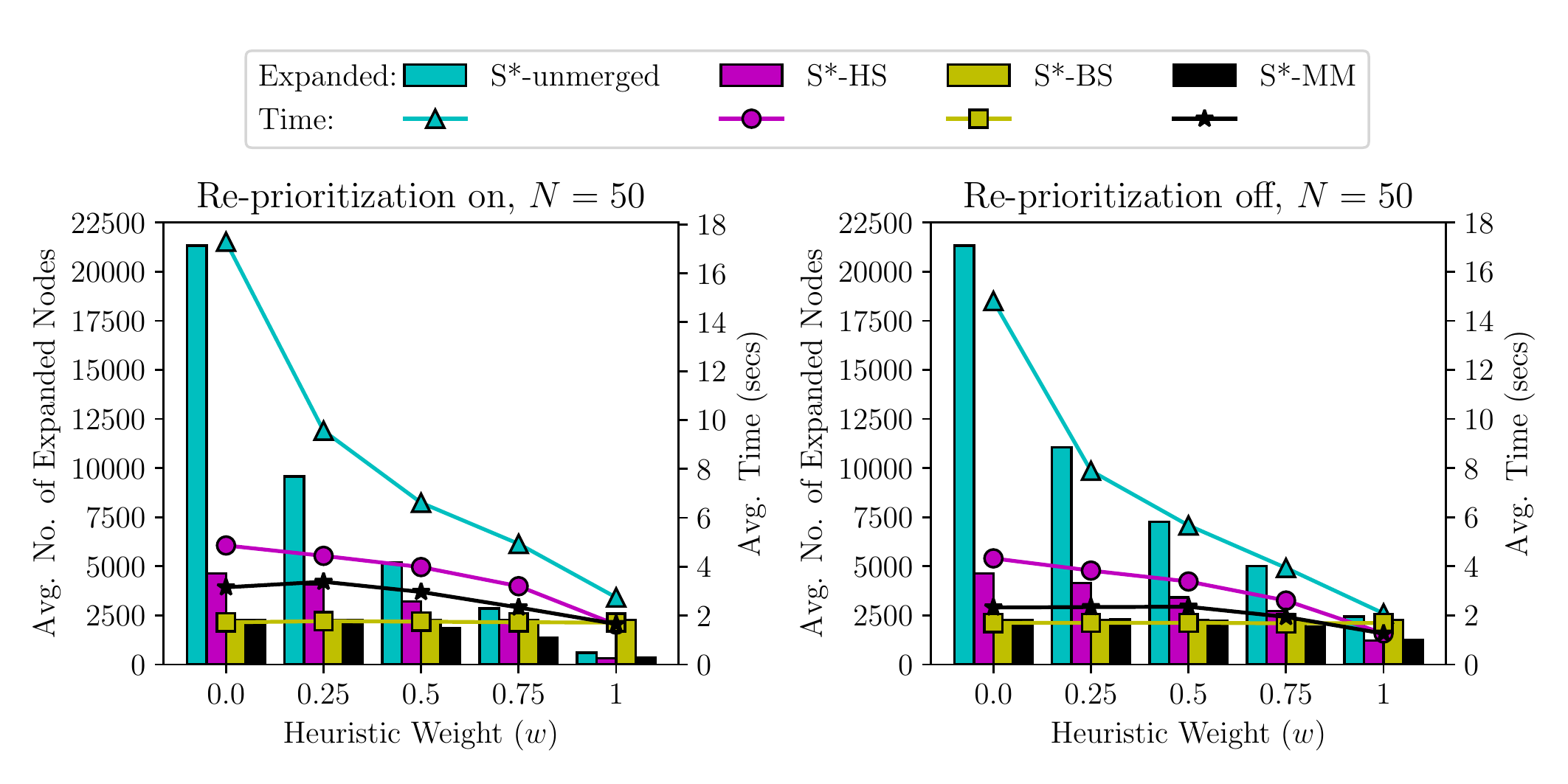} 
\caption{``den312d" map: Average number of expanded nodes and computation times as a function of heuristic strength and re-prioritization. $N$ is fixed at 50.}
\label{fig:res_heuristic}
\end{figure*}

\begin{table*}[ht!]
\centering
\begin{tabular}{@{}cllllll@{}}
\toprule
\multirow{3}{*}{Map}                              & \multirow{3}{*}{Algorithm}     & \multicolumn{5}{c}{Avg. No. of Expanded Nodes (Avg. runtime in secs)}                                              \\ \cmidrule(lr){3-7} 
                                 &           & $w=0$     & $w=0.25$ & $w=0.50$ & $w=0.75$ & $w=1$    \\ \midrule
\multirow{5}{*}{\includegraphics[width=0.25\columnwidth]{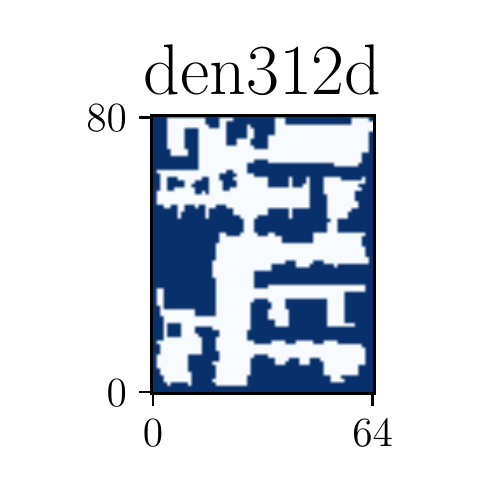}}         & Kruskal   & 119805 (21.29)  & 119805 (21.26) & 119805 (21.26) & 119805 (21.26) & 119805 (21.26) \\
                                 & S*-unmerged  & 21334.6 (14.80) & 11064.8 (7.89) & 7269.6 (5.66)  & 5019.7 (21.26) & 2444.9 (2.07)  \\
                                 & S*-HS  & 4635.2 (4.32)   & 4139.7 (3.82)  & 3412.2 (3.37)  & 2723.2 (2.60)  & 1221.0 (1.26)  \\
                                 & S*-BS   & 2262.6 (1.68)   & 2262.6 (1.69)  & 2262.6 (1.69)  & 2262.6 (1.67)  & 2262.6 (1.68)  \\
                                 & S*-MM     & 2262.6 (2.32)   & 2310.6 (2.33)  & 2223.8 (2.35)  & 1906.7 (1.93)  & 1227.6 (1.25)  \\ \midrule
\multirow{5}{*}{\includegraphics[width=0.2\columnwidth]{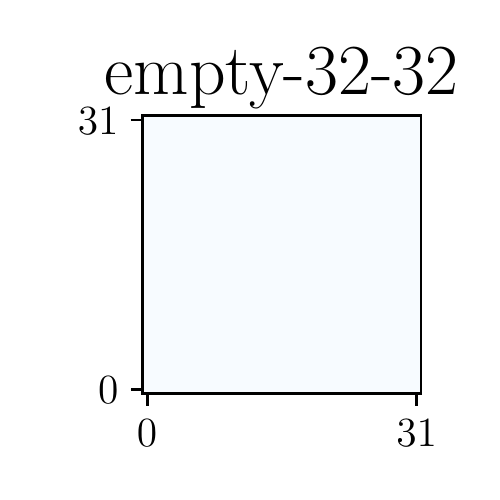}}     & Kruskal   & 50176 (9.08)    & 50176 (9.08)   & 50176 (9.08)   & 50176 (9.08)   & 50176 (9.08)   \\
                                 & S*-unmerged  & 6201.6 (4.52)   & 4022.6 (3.01)  & 2805.6 (2.27)  & 2026.6 (1.69)  & 1159.2 (1.05)  \\
                                 & S*-HS  & 1881.8 (1.84)   & 1657.4 (1.62)  & 1231.2 (1.28)  & 939.9 (0.99)   & 489.2 (0.55)   \\
                                 & S*-BS   & 702.0 (0.55)    & 702.0 (0.57)   & 702.0 (0.56)   & 702.0 (0.55)   & 702.0 (0.56)   \\
                                 & S*-MM     & 702.0 (0.78)    & 714.2 (0.80)   & 776.4 (0.85)   & 702.3 (0.78)   & 492.4 (0.54)   \\ \midrule
\multirow{5}{*}{\includegraphics[width=0.2\columnwidth]{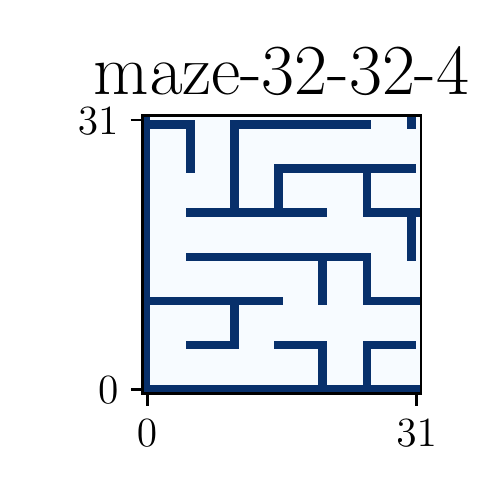}}    & Kruskal   & 38710 (6.70)    & 38710 (6.70)   & 38710 (6.70)   & 38710 (6.70)   & 38710 (6.70)   \\
                                 & S*-unmerged  & 8808.8 (6.17)   & 4740.3 (3.38)  & 3178 (2.40)    & 2361.8 (1.84)  & 1751.9 (1.44)  \\
                                 & S*-HS  & 1461.7 (1.39)   & 1379.7 (1.32)  & 1138.3 (1.14)  & 980.7 (0.97)   & 584.3 (0.59)   \\
                                 & S*-BS   & 768.9 (0.58)    & 768.9 (0.58)   & 768.9 (0.58)   & 768.9 0.58)    & 768.9 (0.59)   \\
                                 & S*-MM     & 768.9 (0.79)    & 775.8 (0.80)   & 759.8 (0.80)   & 710.3 (0.73)   & 583.0 (0.59)   \\ \midrule
\multirow{5}{*}{\includegraphics[width=0.2\columnwidth]{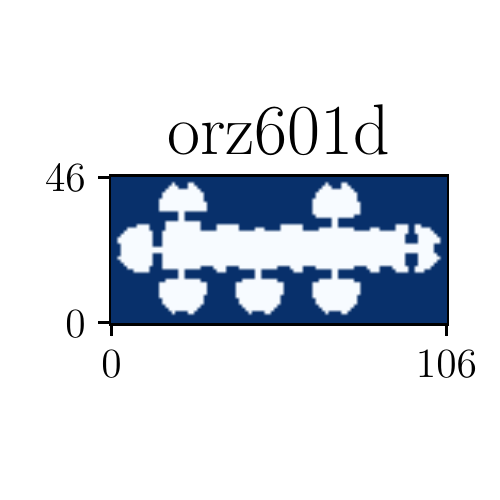}}         & Kruskal   & 92610 (16.48)   & 92610 (16.48)  & 92610 (16.48)  & 92610 (16.48)  & 92610 (16.48)  \\
                                 & S*-unmerged  & 14639.7 (10.42) & 8503.5 (6.26)  & 5745.9 (4.54)  & 3828.9 (3.10)  & 1994.1 (1.76)  \\
                                 & S*-HS  & 3569.2 (3.35)   & 3152.9 (2.99)  & 2515.6 (2.54)  & 1949.4 (1.96)  & 927.3 (0.99)   \\
                                 & S*-BS   & 1672.7 (1.25)   & 1672.7 (1.25)  & 1672.7 (1.26)  & 1672.7 (1.24)  & 1672.7 (1.26)  \\
                                 & S*-MM     & 1672.7 (1.68)   & 1723.5 (1.74)  & 1684.3 (1.26)  & 1424.9 (1.48)  & 931.7 (0.97)   \\ \midrule
\multirow{5}{*}{\includegraphics[width=0.2\columnwidth]{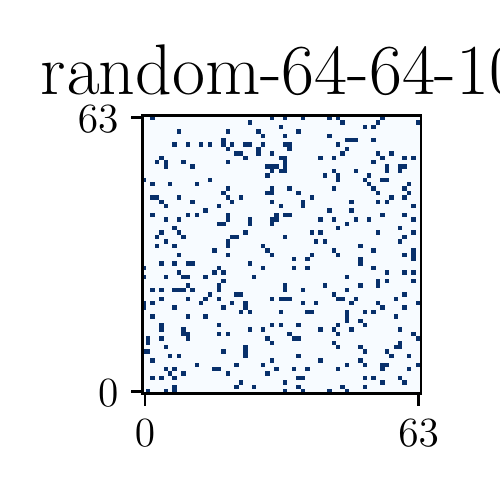}} & Kruskal   & 180663 (31.31)  & 180663 (31.31) & 180663 (31.31) & 180663 (31.31) & 180663 (31.31) \\
                                 & S*-unmerged  & 24270.9 (17.41) & 12976.3 (9.57) & 8055.3 (6.36)  & 4951.1 (2.25)  & 1981 (1.78)    \\
                                 & S*-HS  & 6869.7 (6.45)   & 5779.4 (5.49)  & 4366.9 (4.38)  & 2981.1 (3.01)  & 1111.3 (1.25)  \\
                                 & S*-BS   & 2792.7 (2.07)   & 2792.7 (2.08)  & 2792.7 (2.10)  & 2792.7 (2.09)  & 2792.7 (2.09)  \\
                                 & S*-MM     & 2792.7 (2.87)   & 2899.5 (3.03)  & 2794.2 (2.97)  & 2171.2 (2.25)  & 1137.2 (1.24)  \\ \bottomrule
\end{tabular}
\caption{Summary of results for $N=50$ terminals and varying heuristic weights with no re-prioritization.}
\label{tab:summary}
\end{table*}

The results in Fig. \ref{fig:res_terminals} show that the number of expanded nodes, on average, reduced by nearly 50\% with re-prioritization for algorithms ({\myfont S\textsuperscript{*}-unmerged}, {\myfont S\textsuperscript{*}-HS}, and {\myfont S\textsuperscript{*}-MM}) at the expense of some additional computation time; these reductions also become more pronounced as the number of terminals increased. While re-prioritization did not significantly affect the computation times, the trends show that this will be a factor for a larger number of terminals. This overhead is likely linked to the data structures used for the open sets in the algorithms. Presently, each open-set is implemented using a binary heap based priority queue. More efficient data structures will be investigated in future work. 


\subsection{Impact of the quality of heuristics}

Results are reported here for problem instances with 50 terminals. Fig. \ref{fig:res_heuristic} shows the average number of expanded nodes and computational times for each algorithm as a function of the weighting factor ($w$) used for the heuristics. $w=0$ is equivalent to using no heuristic estimates and $w=1$ corresponds to using the best possible estimates (computed using the landmark based algorithms described earlier). In general, we observed that {\myfont S\textsuperscript{*}-MM} expanded the least number of nodes with lower computational times compared to all the other algorithms (this can also be inferred in Table \ref{tab:summary}). While the algorithms ({\myfont S\textsuperscript{*}-unmerged}, {\myfont S\textsuperscript{*}-HS}, and {\myfont S\textsuperscript{*}-MM}) expanded significantly a fewer number of nodes in comparison to {\myfont S\textsuperscript{*}-BS} when $w=1$ (particularly when re-prioritization is turned on), {\myfont S\textsuperscript{*}-BS} performed better than other algorithms when $w=0$.


There are also subtle differences between the two merged heuristic-based algorithms. Figure \ref{fig:res_heuristic} shows that using stronger heuristics have a greater effect on {\myfont S\textsuperscript{*}-HS} than with {\myfont S\textsuperscript{*}-MM}. This is because 
in general, we observe that MM confirms paths more aggressively than HS especially for less accurate heuristics.
When $w=0$, {\myfont S\textsuperscript{*}-MM} behaves identically to {\myfont S\textsuperscript{*}-BS}, and outperforms {\myfont S\textsuperscript{*}-HS}. When $w=1$, the performance of both  {\myfont S\textsuperscript{*}-HS} and {\myfont S\textsuperscript{*}-MM} are quite similar. 



\subsection{{\textbf{\textit{A-posteriori}}} guarantees of the proposed algorithms}

\begin{table}[]
\centering
\begin{tabular}{@{}clll@{}}
\toprule
Map             & Min   & Avg   & Max   \\ \midrule
den312d         & 1.780 & 1.872 & 1.966 \\
empty-32-32     & 1.788 & 1.900 & 1.976 \\
maze-32-32-4    & 1.680 & 1.846 & 1.977 \\
orz601d         & 1.693 & 1.835 & 1.962 \\
random-64-64-10 & 1.815 & 1.882 & 1.939 \\
 \bottomrule
\end{tabular}
\caption{Minimum, average and maximum {\it a-posteriori} guarantees obtained for all the test instances.}
\label{tab:tree-path-ratio}
\end{table}
 The quality of the solutions obtained by any of the proposed algorithms for MGPF can be inferred by computing the {\it a-posteriori} guarantee, $i.e.$, for a given instance, the {\it a-posteriori} guarantee is defined as the ratio of the cost of the feasible solution obtained by an algorithm and a lower bound to the optimal cost. The minimum, average and maximum {\it a-posteriori} guarantees obtained for the tested instances is shown in Table \ref{tab:tree-path-ratio}. These guarantees are generally lower than the approximation ratio which is a (worst-case) theoretical bound for any instance of the problem. A feasible path is constructed by following the procedure in Fig. \ref{fig:SteinerApprox}. The lower bound to the optimal cost used here is simply the cost of the Steiner tree obtained using any of the proposed algorithms. 

\section{Conclusions}

In this article, a framework called S$^*$ was presented for developing a suite of efficient 2-approximation algorithms for MGPF. Additionally, numerical results were also presented to compare the algorithms from the proposed framework with the conventional solvers in terms of the number of expanded nodes and computation time. Overall, the results show that the version of the proposed framework which uses the MM algorithm \cite{MM} performed the best. Future work can explore decentralized implementations and alternate data structures for faster implementations of S$^*$.



\section{Acknowledgements}

This material is based upon work supported by the Air Force Office of Scientific Research under award number FA9550-20-1-0080, and the Army Research Office under Cooperative Agreement Number W911NF-19-2-0243. The views and conclusions contained in this document are those of the authors and should not be interpreted as representing the official policies, either expressed or implied, of the Army Research Office or the U.S. Government. The U.S. Government is authorized to reproduce and distribute reprints for Government purposes notwithstanding any copyright notation herein.

\bibliography{bib}

\ifthenelse{\boolean{shortver}}{%

}{ 

\onecolumn

\section{Appendix}

Figure \ref{fig:expectation} showed our expected hierarchy on the number of expanded nodes for the algorithms presented in this article. While this expectation is supported by our numerical results, one can certainly develop counterexamples where this hierarchy can be violated. Consider a 2-terminal example with obstacles (Fig.~\ref{fig:app-counter-ex-bare}), where the optimal path cost is 14 units. If a weak heuristic (say Euclidean-distance based) is used, {\myfont S\textsuperscript{*}-HS}, a merged variant (Fig.~\ref{fig:app-counter-ex-hs}), can expand more nodes than primal-dual (Fig.~\ref{fig:app-counter-ex-pd}). On the other hand, if a strong heuristic is used, {\myfont S\textsuperscript{*}-unmerged} (Fig.~\ref{fig:app-counter-ex-um}) can expand fewer nodes than primal-dual.

\begin{figure*}[h]
\centering{\phantomsubcaption\label{fig:app-counter-ex-bare}\phantomsubcaption\label{fig:app-counter-ex-hs}\phantomsubcaption\label{fig:app-counter-ex-pd}\phantomsubcaption\label{fig:app-counter-ex-um}}
\includegraphics[width=1\textwidth]{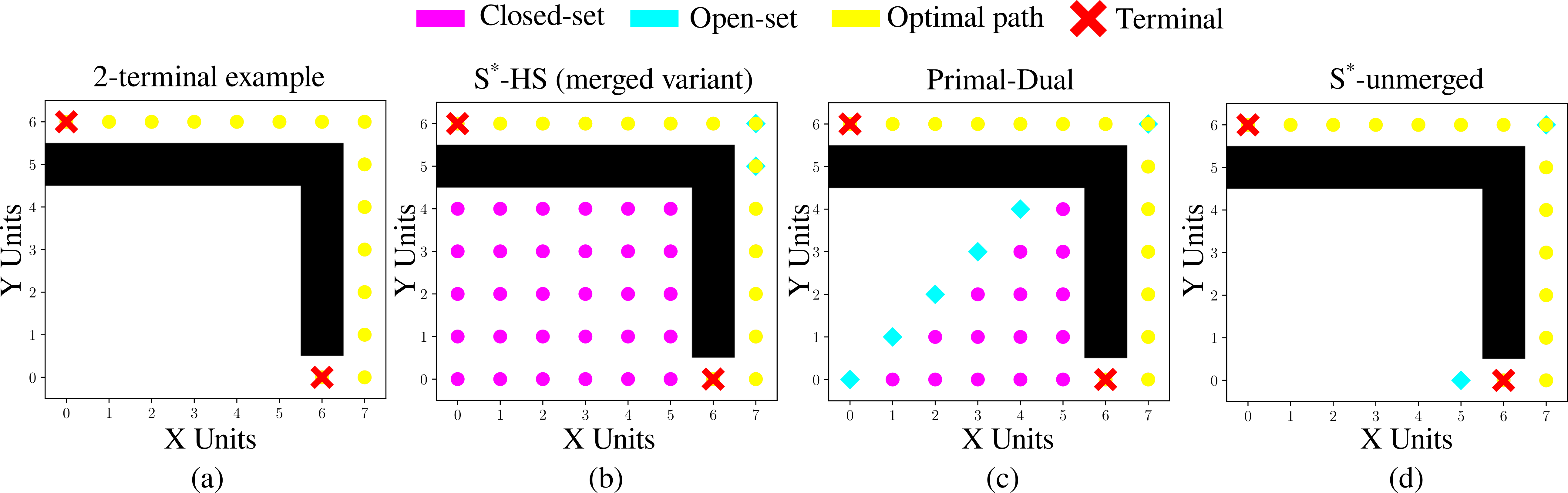} 
\label{fig:app-counter-ex}
\caption{S\textsuperscript{*}-merged/unmerged may expand more or less nodes than primal-dual depending on the heuristic strength. Here, expanded nodes include both the nodes present in the closed set and the optimal path. } 
\end{figure*}

}
\end{document}